\PassOptionsToPackage{compress}{natbib}

\documentclass{l4dc2024}

\usepackage{amsmath,amsfonts,bm}

\def\1{\bm{1}}

\def\vepsilon{{\bm{\epsilon}}}
\def\valpha{{\bm{\alpha}}}

\def\vc{{\bm{c}}}

\def\vf{{\bm{f}}}

\def\vh{{\bm{h}}}

\def\vo{{\bm{o}}}

\def\vq{{\bm{q}}}

\def\vv{{\bm{v}}}
\def\vw{{\bm{w}}}
\def\vx{{\bm{x}}}
\def\vy{{\bm{y}}}
\def\vz{{\bm{z}}}

\def\vsF{{\mathcal{F}}}

\def\vsH{{\mathcal{H}}}

\def\vsX{{\mathcal{X}}}
\def\vsY{{\mathcal{Y}}}
\def\vsZ{{\mathcal{Z}}}

\def\mA{{\bm{A}}}
\def\mB{{\bm{B}}}
\def\mC{{\bm{C}}}

\def\mI{{\bm{I}}}

\def\mK{{\bm{K}}}

\def\mO{{\bm{O}}}

\DeclareMathAlphabet{\mathsfit}{\encodingdefault}{\sfdefault}{m}{sl}
\SetMathAlphabet{\mathsfit}{bold}{\encodingdefault}{\sfdefault}{bx}{n}

\newcommand{\operator}[1]{\mathsf{#1}}
\def\oA{{\operator{A}}}
\def\oB{{\operator{B}}}

\def\oI{{\operator{I}}}

\def\oK{{\operator{K}}}

\def\oO{{\operator{O}}}

\def\oQ{{\operator{Q}}}

\def\oT{{\operator{T}}}
\def\oU{{\operator{U}}}

\def\sD{{\mathbb{D}}}

\def\sI{{\mathbb{I}}}

\newcommand{\R}{\mathbb{R}}

\DeclareMathOperator*{\argmin}{arg\,min}

\usepackage{mathtools}
\usepackage{xparse}
\usepackage{amssymb}
\usepackage{bm}
\usepackage{xfrac}
\usepackage{tensor}
\usepackage{eqparbox}
\usepackage{xcolor}
\usepackage{pifont}
\usepackage[all,cmtip]{xy}
\usepackage{xifthen}
\usepackage{relsize}
\DeclarePairedDelimiterX{\norm}[1]{\lVert}{\rVert}{#1}
\newcommand{\st}{\;|\;}
\newcommand{\stforall}{, \; \forall \; }

\newcommand{\C}{\mathbb{C}} 
\newcommand{\N}{\mathbb{N}} 
\newcommand{\transpose}{\intercal}

\newcommand{\innerprod}[3][]{{\langle #2, #3 \rangle}_{#1}}

\newcommand{\G}[1][]{\mathbb{G}_{\scalebox{0.6}{$#1$}}}       
\newcommand{\KleinFourGroup}{\mathbb{K}_{4}}            
\newcommand{\CyclicGroup}[1][]{\mathbb{C}_{#1}}        
\newcommand{\UGroup}[1][]{\mathbb{U}(#1)}
\newcommand{\GLGroup}[1][]{\mathbb{GL}({#1})}          
 
\newcommand{\g}{g}       
\newcommand{\Glact}{\mathrel{\mathsmaller{\triangleright}}}             
\newcommand{\Gcomp}{\mathrel{\mathsmaller{\circ}}}                      

\newcommand{\mapping}[5]{ 
        \begin{matrix}
            #1:& #2 & \longrightarrow & #3 \\
            & #4    & \longrightarrow & #5
        \end{matrix}
}

\newcommand{\Oplus}{\ensuremath{\vcenter{\hbox{\scalebox{1.1}{$\bm{\oplus}$}}}}}
\newcommand{\irrepMultiplicity}[1][]{m_{#1}}

\newcommand{\rep}[2][]{  
    {\rho_{
        {\vcenter{\hbox{\scalebox{0.7}{$#1$}}}}
    }
    \def\temp{#2}\ifx\temp\empty
    \else
      (#2)%
    \fi
    }
} 
\newcommand{\irrep}[2][]{
    {\bar{\rho}_{\scriptscriptstyle{#1}}
    \def\temp{#2}\ifx\temp\empty
    \else
      (#2)%
    \fi
    }
} 

\newcommand{\homomorphism}[3][\G]{\text{Homo}_{#1}(#2,#3)}
\newcommand{\homomorphismDiag}[6]{
    \xymatrix@C=1.5em{
    #1 \ar@{-}[r]^{#6}    \ar[d]^{#5}    & #2 \ar[d]^{#5} \\
    #3 \ar@{-}[r]^{#6}                   & #4
    }
}
\newcommand{\isomorphism}[3][\G]{\text{Iso}_{#1}(#2,#3)}
\newcommand{\isomorphismDiag}[5]{
    \xymatrix{
        #1 \ar@{-}[r]^{#3}    \ar@{-}[d]^{#5} & #1 \ar@{-}[d]^{#5} \\
        #2 \ar@{-}[r]^{#4}                   & #2
    }
}

\newcommand{\changeOfBasis}[1][]{\oQ_{#1}}     
\newcommand{\isoCompNum}{{k}}
\newcommand{\isoCompIdx}{{i}}
\newcommand{\isoCompIrrepIdx}{{j}}

\newcommand{\HilbertSpace}[1][]{\mathcal{X}_{#1}}
\newcommand{\identity}[1][]{\bm{1}}

\newcommand{\Time}{\mathbb{T}}

\newcommand{\measure}[1][]{\mu_{#1}}     

\newcommand{\equivLinMap}[1][]{\tensor*[]{\oT}{#1}}

\newcommand{\confSpace}{\mathcal{Q}}

\newcommand{\tangConfSpace}[1][\q]{\mathcal{T}_{#1}\confSpace}

\newcommand{\detDynMap}[1][\domain]{\Phi_{\scalebox{0.6}{$#1$}}^{\scalebox{0.7}{$\dt$}}}

\newcommand{\dynRepError}[1][]{\text{err}_{#1}}

\newcommand{\state}{\omega} 
\newcommand{\predState}{\tilde{\omega}}         
\newcommand{\domain}{{\Omega}}        
\newcommand{\basisSet}{\sI}
\newcommand{\dt}{{\mathrel{\mathsmaller{\Delta}t}}}

\newcommand{\dimModelSpace}{m}
\newcommand{\dimLatModelSpace}{\ell}

\newcommand{\q}[1][]{\vq_{#1}}                    
\newcommand{\dq}[1][]{\dot{\vq}_{#1}}                   

\newcommand{\obsSpace}[1][]{
    \ifthenelse{\isempty{#1}}
    {
                \vsF
    } 
    {
        \scalebox{0.9}{
        $
                \vsF_{#1}
        $        
        }    
    } 
}
\newcommand{\approxObsSpace}[1][]{
    \ifthenelse{\isempty{#1}}
    {
            \widetilde{\vsF}
    } 
    {
        \scalebox{0.83}{
            $
            \tensor*[]{\widetilde{\vsF}}{#1}
            $
            }
    } 
}

\newcommand{\equivObsSpaceDual}[1][_{\measure[t]}^{*\G}]
{
    \tensor*[]{\scalebox{0.8}{$\mathcal{X}$}}{#1}
}

\newcommand{\obsState}[1][]{\vx_{#1}}               
\newcommand{\latObsState}[1][]{\vz_{#1}}               
\newcommand{\obsFn}{x}
\newcommand{\repObsFn}{\valpha}                  
\newcommand{\latObsFn}{z}

\newcommand{\encoder}[1][\nnParams]{\nn[#1]}
\newcommand{\decoder}[1][\nnParams]{\nn[#1]^{\mathrel{\mathsmaller{-1}}}}

\newcommand{\nn}[1][\nnParams]{\vf_{#1}}
\newcommand{\loss}{\mathcal{L}}

\newtheorem{theorem}{Theorem}

\newtheorem{lemma}{Lemma}

\newtheorem{definition}{Definition}      
\usepackage{tcolorbox} 
\usepackage{xcolor}
\usepackage{physics}
\usepackage{amsmath}
\usepackage{float}
\usepackage{tikz}
\usepackage{mathdots}
\usepackage{yhmath}
\usepackage{cancel}
\usepackage{color}
\usepackage{siunitx}
\usepackage{array}
\usepackage{multirow}
\usepackage{amssymb}
\usepackage{gensymb}
\usepackage{extarrows}
\usepackage{booktabs}

\usepackage{hyperref}
\usepackage[nameinlink]{cleveref}   

\usepackage{graphicx}   
\usepackage{ulem}       
\usepackage{mathtools}  
\usepackage{multicol}   
\usepackage{adjustbox}  
\usepackage{wrapfig}    

\usepackage{enumitem}       
\usepackage{pifont}         

\setlist{leftmargin=5.5mm}  
\usetikzlibrary{fadings}
\usetikzlibrary{patterns}
\usetikzlibrary{shadows.blur}
\usetikzlibrary{shapes}

\definecolor{gray}{rgb}{0.6, 0.7, 0.7}
\definecolor{awesomeblue}{rgb}{0.054, 0.415, 0.505}
\definecolor{awesomeorange}{rgb}{0.570, 0.458, 0.0912}

\newcommand{\highlight}[1]{{\textit{#1}}}

\newtcolorbox{highlightBox}[3][]
{
  colframe = #2!25,
  colback  = #2!15,
  coltitle = #2!20!black,  
  title    = {\textbf{#3}},
  #1,
}

\newcommand{\ubcolor}[3][awesomeblue]{{
        \color{#1}{
            \underbrace{\color{black}{#2}}_{#3}
        }
    }}

\crefname{section}{sec}{sec}    

\usepackage{scalerel}
\usepackage{nicefrac}
\usepackage{subfiles}
\usepackage[toc,page]{appendix}

\usepackage[]{todonotes}

\usepackage{caption}
\crefname{definition}{def.}{defs.}
\crefname{proposition}{prop.}{props.}
\crefname{theorem}{thm.}{thms.}
\crefname{figure}{fig.}{figs.}
\crefname{section}{sec.}{secs.}
\crefname{appendix}{appendix}{appendix}

\title[Dynamics Harmonic Analysis of Robotic Systems]{Dynamics Harmonic Analysis of Robotic Systems: \\Application in Data-Driven Koopman Modelling}
\usepackage{times}

\coltauthor{
 \Name{Daniel Felipe {Ordoñez-Apraez}}
    \textsuperscript{\normalfont 1,2} 
    \Email{daniel.ordonez@iit.it}\\
 \Name{Vladimir Kostic}
    \textsuperscript{\normalfont 1} 
    \Email{vladimir.kostic@iit.it}\\
 \Name{Giulio Turrisi}
    \textsuperscript{\normalfont 2} 
    \Email{giulio.turrisi@iit.it}\\
 \Name{Pietro Novelli}
    \textsuperscript{\normalfont 1} 
    \Email{pietro.novelli@iit.it}\\
 \Name{Carlos Mastalli}
    \textsuperscript{\normalfont 3} 
    \footnote{Also with the Florida Institute for Human and Machine Cognition, US.} 
    \Email{c.mastalli@hw.ac.uk}
    \\
 \Name{Claudio Semini} 
    \textsuperscript{\normalfont 2} 
    \Email{claudio.semini@iit.it}\\
 \Name{Massimiliano Pontil}
    \textsuperscript{\normalfont 1}
    \footnote{Also with the Department of Computer Science, University College London, UK.} \Email{massimiliano.pontil@iit.it}\\
 \addr 
 \textsuperscript{\normalfont 1}
 Computational Statistics and Machine Learning - Istituto Italiano di Tecnologia \\
 \addr 
 \textsuperscript{\normalfont 2}
 Dynamic Legged Systems - Istituto Italiano di Tecnologia \\ 
 \addr 
 \textsuperscript{\normalfont 3}
 Robot Motor Intelligence -- Heriot-Watt University
 }

\begin{document}

\maketitle
\vspace*{-2em}
\begin{abstract}%
    \noindent
    We introduce the use of harmonic analysis to decompose the state space of symmetric robotic systems into orthogonal isotypic subspaces. These are lower-dimensional spaces that capture distinct, symmetric, and synergistic motions. For linear dynamics, we characterize how this decomposition leads to a subdivision of the dynamics into independent linear systems on each subspace, a property we term dynamics harmonic analysis (DHA). To exploit this property, we use Koopman operator theory to propose an equivariant deep-learning architecture that leverages the properties of DHA to learn a global linear model of the system dynamics. Our architecture, validated on synthetic systems and the dynamics of locomotion of a quadrupedal robot, exhibits enhanced generalization, sample efficiency, and interpretability, with fewer trainable parameters and computational costs.

\end{abstract}


\section{Introduction}
The current state-of-the-art in modelling, control, and estimation of robotic systems relies on the Lagrangian model of rigid-body dynamics, which represents the system's state as a point in the space of generalized (or minimal) coordinates. This approach has fostered the development of efficient algorithms that leverage the system's kinematic structure to enable recursive computations which are ubiquitous in methods for simulation, estimation, planning, and control in robotics ~\citep{featherstone-rbdbook}. However, because the model's dynamics are nonlinear, these methods need to cope with the challenges of nonlinear optimization, often through iterative local (or state-dependent) linearizations \citep{mayne1966DDP,li2004iLQR}. This can limit control policies to local minima, hinder the convergence of optimization methods, and bias the estimation of unobserved quantities.

The Koopman operator framework can potentially address these limitations by deriving globally linear models of robot dynamics, albeit in an infinite-dimensional function space \citep{brunton-koopman,kostic2023deep}.
These models can be easily used with various estimation and control theory techniques \citep{mauroy2020koopman} and can capture dynamic phenomena that impact the system's evolution but are challenging to model analytically \citep{asada2023global}. However, building a robust data-driven model approximation in finite dimensions is a substantial machine-learning challenge. 
This study highlights the value of leveraging the state symmetries inherent in a system's dynamics as a geometric prior, enhancing the approximation of the operator. This approach improves sample efficiency, generalization, and interpretability while reducing the number of trainable parameters and computational cost. Although our emphasis is on robotics, the proposed framework is broadly applicable to any dynamical system with discrete symmetry groups.

\begin{figure}[t!]
    \centering
    \includegraphics[width=\linewidth]{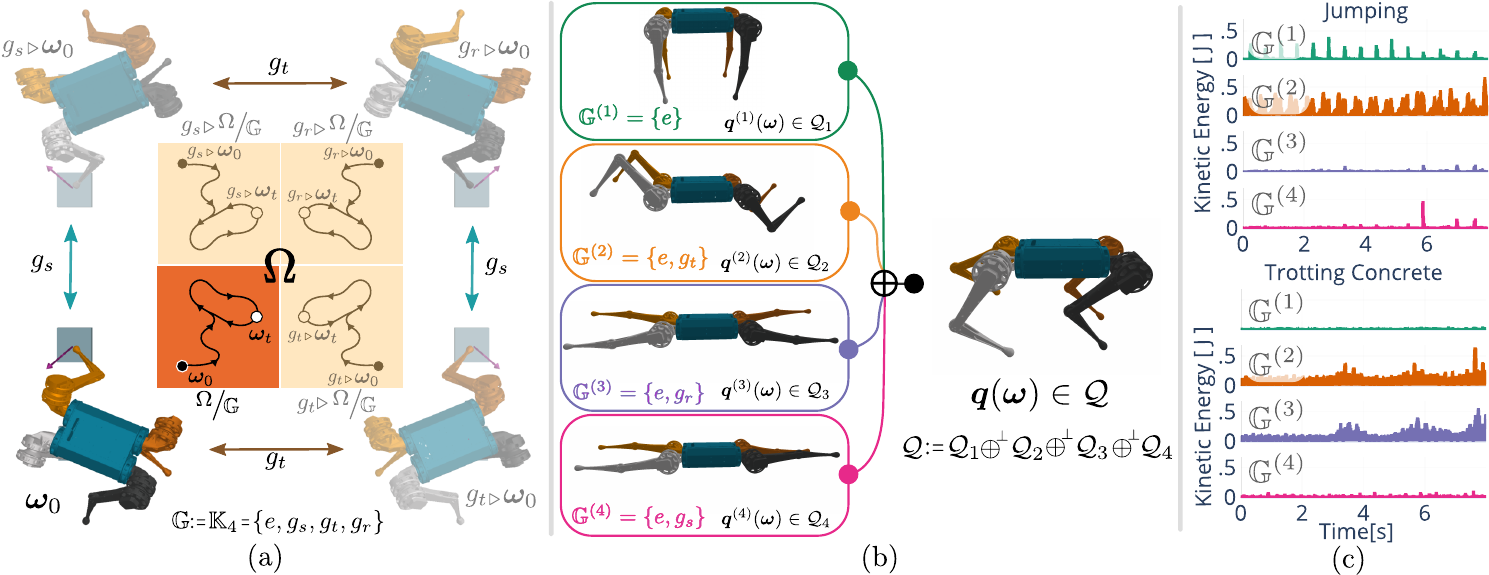}
    \caption{
        \small
        (a) Diagram of the discrete symmetry group $\G:=\KleinFourGroup$ of the mini-cheetah robot (see \href{https://github.com/Danfoa/MorphoSymm/blob/devel/docs/static/animations/mini_cheetah-Klein4-symmetries_anim_static.gif?utm_source=l4dc&utm_medium=l4dc&utm_campaign=l4dc}{animation}).
        Each symmetry $\g\in \G$ relates states that evolve identically under physics laws; see \href{https://github.com/Danfoa/MorphoSymm/blob/devel/docs/static/dynamic_animations/mini-cheetah-dynamic_symmetries_forces.gif?utm_source=l4dc&utm_medium=l4dc&utm_campaign=l4dc}{evolution} of states related by $\g_{s}$.
        This results in the decomposition of the set of states $\domain$ into $\g$-transformed copies of the quotient set $\sfrac{\domain}{\G}$, encompassing all unique system states.
        (b) Isotypic decomposition of the robot's space of generalized position coordinates $\confSpace$ into isotypic subspaces:
        $
            \confSpace:= \Oplus_{\isoCompIdx=1}^{4} \confSpace_\isoCompIdx
        $. Each subspace, $\confSpace_\isoCompIdx$ describes a space of symmetry-constrained synergistic motions. Consequently, any position configuration $\q(\state) \in \confSpace$, can be decomposed into projections within these subspaces:
        $
            \q(\state) := \Oplus_{\isoCompIdx=1}^{4} \q^{(i)}(\state)
        $
        (see \href{https://github.com/Danfoa/DynamicsHarmonicsAnalysis/blob/main/media/DynamicsHarmonicAnalysis_mini_cheetah_K4.md?utm_source=l4dc&utm_medium=l4dc&utm_campaign=l4dc}{animation}).
        (c) Joint-space Kinetic energy distribution across isotypic subspaces for two gait/motion trajectories in the real world: jumping and trotting (see \href{https://github.com/Danfoa/DynamicsHarmonicsAnalysis/blob/main/media/DynamicsHarmonicAnalysis_mini_cheetah_K4.md?utm_source=l4dc&utm_medium=l4dc&utm_campaign=l4dc}{animation}). Both gaits primarily evolve within one or two lower-dimensional isotypic subspaces, with less significant subspaces engaged temporarily during disturbances.
    }
    \vspace*{-0.6em}
    \label{fig:teaser}
\end{figure}

In robotics, the symmetries we aim to exploit are known as morphological symmetries \citep{Ordonez-Apraez-RSS-23}. These are discrete symmetry groups that capture the equivariance of the robot's dynamics, arising from the duplication of rigid bodies and kinematic chains (see \cref{fig:teaser}-a). 
Inspired by the use of harmonic analysis of discrete symmetry groups in physics~\citep{dresselhaus2007group_theory_applications_to_physics_of_condensed_matter}, we present the first application of these principles in robotics. Specifically, we employ the isotypic decomposition (\cref{thm:iso}) to partition the state space of symmetric robotic systems into isotypic subspaces (see \cref{fig:teaser}-b). This allows us to represent any system motion (e.g., different locomotion gaits or manipulation movements) as a superposition of simpler, symmetric, synergistic motions\footnote{
    These are the robotics analog of normal vibrational modes in molecular dynamics; see \citep[8.3]{dresselhaus2007group_theory_applications_to_physics_of_condensed_matter}
}; each evolving in a distinct isotypic subspace (see \cref{fig:teaser}-c).
Moreover, we demonstrate that the isotypic decomposition leads to partitioning any (local or global) linear model of the system's dynamics into independent linear subsystems, each characterizing the dynamics of an isotypic subspace. This offers numerous computational advantages and enhances the interpretability of the dynamics.
%
\paragraph{Contributions} In summary, our work relies on the following contributions: (i) we introduce the use of harmonic analysis of discrete symmetry groups to robotics; (ii) we propose the concept of dynamics harmonic analysis (DHA), illustrating how local/global linear dynamics models decompose into linear subsystems within isotypic subspaces (\cref{sec:symmetries_dynamical_systems}); (iii) leveraging on DHA, we propose the equivariant Dynamics Autoencoder (eDAE), a deep-learning architecture to approximate the Koopman operator (\cref{sec:method}), and report its strong performance on synthetic and robotics systems (\cref{sec:experiments}); and (iv) we provide an open-access \href{https://danfoa.github.io/DynamicsHarmonicsAnalysis/}{repository} enabling the application of harmonic analysis to a \href{https://github.com/Danfoa/MorphoSymm/tree/devel?tab=readme-ov-file#library-of-symmetric-dynamical-systems}{library} of symmetric robots,  the use of the eDAE architecture, and our experiments' reproduction.
\paragraph{Related work}
For a data-driven approximation of Koopman operators \citet{noe2013vamp}, \citet{lusch2018deep}, and \citet{kostic2023deep} introduce symmetry-agnostic deep learning algorithms, which can be adapted to exploit DHA, as proposed in \cref{sec:method}. In the spirit of our work, \cite{steyert2022uncovering_structure} studies the operator's structure for systems evolving on manifolds, i.e., featuring continuous symmetry groups. For discrete symmetry groups, \citet{salova2019koopman} leverages the operator's block-diagonal structure after using harmonic analysis on a non-learnable dictionary of observable functions to model Duffing oscillators. Lastly, \cite{mesbahi2019modal, sinha2020koopman} provides a theoretical analysis of the operator's structure to model symmetric dynamical systems.

In robotics, the linear nature of Koopman operator models makes them compatible with standard modelling, estimation, and control algorithms. This includes optimal and robust control \citep{korda2018linear,folkestad2021koopmanNMPC,zhang2022robust}, active learning \citep{abraham2019active_learning_koopman}, and system identification and observer synthesis \citep{bruder2020data,surana2020koopman}. Yet, symmetries remain unexploited in these Koopman-based approaches.

\section{Preliminaries}
\label{sec:preliminaries}

This section overviews background material on dynamical systems modelling, Koopman models, and harmonic analysis needed to address the modelling of symmetric dynamical systems in \cref{sec:symmetries_dynamical_systems}.
\paragraph{Modelling of dynamical systems}
 
In our analysis, distinguishing between a dynamical system and its numerical model is crucial. A dynamical system abstracts evolving real-world phenomena, such as a robot's motion in an environment. Conversely, numerical models \highlight{approximate} the system's dynamics as the time evolution of points (representing the system's state) in a vector space. A dynamical system is typically denoted by a tuple $(\domain, \Time, \detDynMap)$, where $\domain$ is the abstract set of system states $\state \in \domain$, the set $\Time$ represents time, and $\detDynMap\colon \domain\cross\Time\mapsto\domain$ is the evolution map, determining the future state based on the present state and time.  Given deterministic Newtonian mechanics governs our target systems, we focus on deterministic, Markovian, discrete-time systems. Here, $\Time = \N_0$ and $\detDynMap$ depends solely on the time difference between consecutive timesteps $\dt$, i.e., $\state_{t+\dt} := \detDynMap(\state_{t})$ for any $t \in \Time, \state \in \domain$. Thus, we will denote a dynamical system by 
 $(\domain, \detDynMap)$.

Defining a numerical model of a dynamical system involves identifying a state representation vector-valued function $\obsState = [\obsFn_{1}, \dots, \obsFn_{\dimModelSpace}]: \domain \rightarrow \vsX \subseteq \R^{\dimModelSpace}$, where the components, $\obsFn_{j}: \domain \rightarrow \R, j\in[1, \dimModelSpace]$, are observable functions that measure a relevant scalar quantity from the state (e.g., kinetic energy, joint position/velocity). 
This enables the representation of the state as a point in the model's vector space $\obsState(\state) \in \vsX$.
The system's evolution, represented as a trajectory $(\obsState(\state_{t}))_{t \in \Time}$, can be approximated by an evolution map $\obsState(\predState_{t+\dt}) := \detDynMap[\vsX](\obsState(\state_{t})) + \kappa(\state_t) \Gamma_t$. Here, the predicted state representation $\obsState(\predState_{t+\dt})$ may differ from the true one $\obsState(\state_{t+\dt})$ due to modelling errors and/or inaccessible observables. The influence of these factors is assumed to be captured by a mild white noise stochastic perturbation $\Gamma_t$, scaled by $\kappa(\state_t)$ \citep[chpt. 10.5]{Lasota1994}. When $\obsState: \domain \to \vsX$ is injective, the optimality of the model $(\vsX, \detDynMap[\vsX])$, over a given prediction horizon $H \in \Time$, is quantified by its predictive error \citep{mezic2021koopman_geometry}, i.e.,

\begin{equation}\label{eq:modelling_error}
    \small
    \dynRepError[H](\obsState, \detDynMap[\vsX]) = \int_\domain \dynRepError[{\state_{0}, H}](\obsState, \detDynMap[\vsX]) d\state_{0}, 
    \quad 
    \text{given} 
    \quad   
    \dynRepError[{\state_{0}, H}](\obsState, \detDynMap[\vsX])
    := 
    \sum_{h=1}^H ||\obsState(\predState_{h\dt}) - \obsState(\state_{h\dt})||^2.
\end{equation}

\noindent
Note that while modelling error can be locally minimal for specific $\state_{0}$ (r.h.s \cref{eq:modelling_error}), optimal models exhibit uniformly small errors across all states and horizons (l.h.s \cref{eq:modelling_error}).
%
\paragraph{Linear models and the Koopman operator}
%
A linear model $(\vsZ,\mK_\dt)$ is a dynamics model where the temporal evolution of state representations is characterized by an autonomous linear system,
$
\latObsState(\predState_{t+\dt}) = \detDynMap[\vsZ](\latObsState(\state)):= \mK_\dt \latObsState(\state),  \state \in \domain, \dt \in \Time
$. Here, $\latObsState = [\latObsFn_1, \dots]: \domain \mapsto \vsZ$ is the state representation function, and 
$
    \mK_\dt\colon\vsZ\mapsto\vsZ
$ is a matrix evolving state representations in time by $\dt$. 
The linearity of these models ensures the existence of analytical solutions to the temporal dynamics of each state observable and guarantees the interpretability of the predictions through modal decomposition, making them fundamental to dynamical systems and control theory. 

Although most dynamical systems of interest in robotics have been historically modeled with nonlinear analytic dynamics, one can also devise an optimal linear model in an infinite-dimensional space. This idea, originating from the seminal work of Koopman and Markov \citep{Lasota1994},
proposes to represent the state as a point/function in a \textit{space of functions} $\obsSpace[\vsZ]$, and model the dynamics with a linear operator $\oK_\dt$, defined by the flow $\detDynMap[\domain]$ and a time step $\Delta t\in \Time$, which takes any function $\latObsFn(\cdot) \in \obsSpace[\vsZ]$ to $\latObsFn(\detDynMap[\domain](\cdot, \Delta t))$. Whenever the image of $\oK_\dt$ is in the same space, $\oK_\dt: \obsSpace[\vsZ] \mapsto \obsSpace[\vsZ]$ is a well-defined linear operator known as the \textit{Koopman operator}, defined by
%
\begin{equation}\label{eq:koopman}
    [\oK_\dt \, \latObsFn](\state_t) := \latObsFn(\detDynMap[\domain](\state_t)), \, \latObsFn \in \obsSpace[\vsZ], \, \state\in \domain.
\end{equation}

\noindent
The requirement for the space $\obsSpace[\vsZ]$ to be invariant under the flow $\detDynMap$ is the characteristic that often renders it infinite-dimensional. As we discuss in \cref{sec:method}, machine learning can aid in finding finite-dimensional approximations of $\obsSpace[\vsZ]$ and $\oK_\dt$ \citep{lusch2018deep,kostic2023deep}.
\paragraph{Symmetry groups and their representations}
In the context of a dynamical systems, symmetries are defined as bijections on the state set $\domain$. The action of a symmetry transformation $\g$ on any state $\state \in \domain$ is defined by a map $(\,\Glact\,): \G \times \domain \rightarrow \domain$ to a symmetric state $\g \Glact \state \in \domain$ (see \cref{fig:teaser}-a). A set of symmetry transformations forms a group $\G=\{e, \g_{1}, \g_{2}, \dots\}$, that is closed under {\textit{composition}:} $\g_{1} \Gcomp \g_{2} \in \G$ for all $\g_{1}, \g_{2} \in \G$, and {\textit{inversion}:} $\g^{-1} \in \G \st \g \in \G$ such that $\g^{-1}  \Gcomp g = e$, where $e$ denotes the identity transformation and $(\Gcomp)$ is the binary composition operation on $\G$.

As we study symmetries of numerical models in both finite-dimensional (Euclidean) spaces and infinite-dimensional function spaces, we assume $\vsH$ to be a separable Hilbert space to accommodate both scenarios. This enable us to rely on the conventional concepts of inner product, orthogonality, and countably many basis elements (see \cref{appendix:observable_functions_and_numerical_representation}). Consequently, symmetry transformations on $\vsH$ are defined via a unitary group representation $\rep[\vsH]{}: \G \rightarrow \UGroup[\vsH]$, mapping each $\g\in \G$ to a unitary matrix/operator $\rep[\vsH]{\g} \in \UGroup[\vsH]: \vsH \rightarrow \vsH$. Thus, the action of any $\g \in \G$ on a point $\vh \in \vsH$ is expressed as $\g \Glact \vh := \rep[\vsH]{\g} \vh \in \vsH$. When $\rep[\vsH]{}$ exists, we say that $\vsH$ is a symmetric space.


A map $f:\vsH \mapsto \vsH'$ between two symmetric spaces is denoted $\G$-equivariant if $f(\rep[\vsH]{\g} \vh) = {\rep[\vsH']{g}} f(\vh)$, and $\G$-invariant if $f(\vh) = f(\rep[\vsH]{\g} \vh)$, for all $\g \in \G$. While the action of a symmetry transformation on $\vsH$ is basis independent, the representation of $ \g\in \G$ depends on the chosen basis of $\vsH$. Therefore, applying a change of basis $\changeOfBasis \in \UGroup[\vsH]$ results in a new point representation $\vh_{\circ} = \changeOfBasis \vh$, and a new group representation $\rho_{\vsH,\circ}(\cdot) = \changeOfBasis \rep[\vsH]{}(\cdot)\changeOfBasis^*$. Here, both representations describe the same transformation: $\rep[\vsH]{g}\vh \iff \rho_{\vsH,\circ}(\g) \vh_{\circ}$. Consequently, representations related by a basis change are termed equivalent, denoted $\rep[\vsH]{} \sim \rho_{\vsH,\circ}$. Lastly, writing $\rep[\vsH]{} \sim \rep[\vsH_1]{} \oplus \rep[\vsH_2]{}$ implies that $\vsH$ decomposes into orthogonal subspaces $\vsH_1$ and $\vsH_2$, and $\rep[\vsH]{}$ has block-diagonal structure.



\paragraph{Isotypic decomposition and its basis}
Our use of harmonic analysis is linked with the decision to work in a specific basis for the modelling space $\vsH$, the \highlight{isotypic basis}. In this basis, the unitary group representations $\rep[\vsH]{}$ decomposes into a block-diagonal sum of multiple copies (\highlight{multiplicities}) of the group's $\isoCompNum$ unique irreducible representations (\textsl{irrep}) $\{\irrep[\isoCompIdx]{}\}_{\isoCompIdx=1}^{\isoCompNum}$. These are the indivisible building blocks of any group representation of $\G$.
Each $
    \irrep[\isoCompIdx]{}: \G \rightarrow \UGroup[{\bar{\vsH}_{\isoCompIdx}}]
$, describes a unique symmetry pattern, characterized by a subset of symmetry transformations within the broader group structure. The space ${\bar{\vsH}_{\isoCompIdx}}$ associated to each \textsl{irrep} is the smallest finite-dimensional space capable of expressing the \textsl{irrep} symmetry pattern. For instance, our groups $\G$ are often subgroups of the orthogonal group. Thus, we frequently work with \textsl{irreps} $\irrep[tr]{}$ that describe a reflection symmetry, requiring a $1$-dimensional space $\bar{\vsH}_{\tr} \sim \R$ to act on, or with \textsl{irreps} $\irrep[\nicefrac{2\pi}{a}]{}$ that describe rotations by an angle $\nicefrac{2\pi}{a}$, requiring a $2$-dimensional space $\bar{\vsH}_{\nicefrac{2\pi}{a}} \sim \R^2 \sim \C$ to act on.  These representations are called irreducible because the spaces ${\bar{\vsH}_{\isoCompIdx}}$ have no non-trivial invariant subspace to the actions of $\G$. That is, if $\mathcal{V}$ is a subspace of $\bar{\vsH}_{\isoCompIdx}$ and $\irrep[\isoCompIdx]{\g} \mathcal{V} \subset \mathcal{V} $ for every $\g\in \G$, then either $\mathcal{V}=\{0\}$ or $\mathcal{V}=\bar{\vsH}_{\isoCompIdx}$.

The value of the isotypic basis lies in the fact that it allows decomposing the modelling space $\vsH$ into an orthogonal sum of isotypic subspaces. Each subspace reflects a unique symmetry pattern of one of the group's \textsl{irreps}, hence the term \textit{iso-typic} or \textit{same-type}. This is a pivotal result in abstract harmonic analysis, succinctly captured by the Peter-Weyl  Theorem, see \citep[thm 1.12]{Knapp1986} and \citep[Thm-2.5]{golubitsky2012singularities_groups_bifurcation}.
%
\begin{theorem}[Isotypic Decomposition]
    \label{thm:iso}
    Let $\G$ be a compact symmetry group and $\vsH$ a symmetric separable Hilbert space with a unitary group representation $\rep[\vsH]{}: \G \rightarrow \UGroup[\vsH]$. Then we can identify $\isoCompNum \leq |\G|$ irreducible representations $\irrep[\isoCompIdx]{}: \G \rightarrow \UGroup[{\bar{\vsH}_{\isoCompIdx}}]$ and change of basis $\oQ \in \UGroup[\vsH]$ such that $
    \rep[\vsH]{}
    \sim
    \rep[{\vsH_{1}}]{} \oplus \rep[{\vsH_{2}}]{} \oplus \cdots \oplus \rep[{\vsH_{\isoCompNum}}]{}
    $, and each $\rep[{\vsH_{\isoCompIdx}}]{} \sim \bigoplus_{\isoCompIrrepIdx=1}^{\irrepMultiplicity[\isoCompIdx]} \irrep[\isoCompIdx]{}$ is composed of at most $\irrepMultiplicity[\isoCompIdx]$ countably many copies of the irreducible representation $\irrep[\isoCompIdx]{}$. This allows to decompose $\vsH$ into orthogonal subspaces: $\vsH = \vsH_{1} \oplus^\perp \vsH_{2} \oplus^\perp \cdots \oplus^\perp \vsH_{\isoCompNum}$. Where each $\vsH_{\isoCompIdx}:= \bigoplus_{\isoCompIrrepIdx=1}^{\irrepMultiplicity[\isoCompIdx]} \vsH_{\isoCompIdx,\isoCompIrrepIdx} $, composed of $\irrepMultiplicity[\isoCompIdx]$ subspaces isometrically isomorphic to ${\bar{\vsH}_{\isoCompIdx}}$, is denoted as an isotypic subspace.
\end{theorem}

\section{Symmetries of dynamical systems}
\label{sec:symmetries_dynamical_systems}
In the context of dynamical systems, a symmetry is a state transformation that results in another functionally equivalent state under the governing dynamics.
From a modelling perspective, symmetries provide a valuable geometric bias, as identifying the dynamics of a single state suffices to capture the dynamics of all of its symmetric states (see \cref{fig:teaser}-a).

\begin{definition}[Symmetric dynamical systems]
    A dynamical system $(\domain,\detDynMap[\domain])$ is $\G$-symmetric, if $\G$ is a symmetry group of the set of states $\domain$, and the system's evolution map is $\G$-equivariant, i.e.,
    \begin{equation}
        \small
        \detDynMap[\domain](\g \Glact \state,t) = \g \Glact \detDynMap[\domain](\state, t), \quad \forall \; \g \in \G, t \in \Time, \state \in \domain.
        \label[equation]{eq:symmetric_dynamical_system}
    \end{equation}
    \vspace*{-2.1em}
    \label[definition]{def:symmetric_dynamical_system}
\end{definition}

\noindent
The symmetry group of $\domain$ defines an equivalence relationship between any state $\state \in \domain$ and its set of symmetric states, denoted $\G \state = \{\g \Glact \state | \g \in \G\}$. Given the $\G$-equivariance of the dynamics $\detDynMap[\domain]$, this equivalence implies that a set of symmetric states $\G \state$ will evolve along a unique trajectory of motion, up to a symmetry transformation $\g \in \G$ (see \cref{fig:teaser}-a). When the symmetry group is discrete (or finite), the state set $\domain$ decomposes into a union of symmetry-transformed copies of the quotient set $\sfrac{\domain}{\G}$, containing the system's unique states, that is $\domain = \cup_{\g\in \G}\{\g \, \Glact \, \sfrac{\domain}{\G} \}$ (see \cref{fig:teaser}-a).

\paragraph{Modelling symmetric dynamical systems}
When designing a numerical model $ \detDynMap[\vsX]$ for a symmetric dynamical system, it is crucial to ensure that the modelling space $\vsX$ inherits the group structure of $\domain$. This can be achieved by making the space invariant under the action of $\G$, i.e., $\obsState(\g\Glact(\cdot))\in \vsX$ for all $\g\in \G$. Thus, enabling the existence of a group representation $\rep[\vsX]{}: \G \rightarrow \UGroup[\vsX]$. The significance of this design choice lies in the fact that the equivalence between symmetric states is translated into the corresponding equivalence of their representations in the modelling space:
\begin{equation}
    \small
    \G \state 
    := 
    \{ \g \Glact \state \stforall \g \in \G \} 
    \iff 
    \G\obsState(\state) 
    := 
    \{ \g \Glact \obsState(\state) := \rep[\vsX]{\g}\obsState(\state) = \obsState(\g \Glact \state) \stforall \g \in \G \}, \;\forall\; \state \in \domain 
    \label{eq:symmetric_modelling_space}
\end{equation}
The symmetric structure of $\vsX$ allows its decomposition into $\g$-transformed copies of a quotient space $\sfrac{\vsX}{\G}$. This is a practical tool in data-driven applications, mitigating the effects of the curse of dimensionality \citep{higgins2022symmetry}. As the following result shows, it also narrows the search space for the evolution map $\detDynMap[\vsX]$ to the space of $\G$-equivariant ones.
\begin{proposition}[Optimal models of $\G$-symmetric systems]
    Let $\detDynMap[\domain]$ be a $\G$-symmetric dynamical system and $\detDynMap[\vsX]$ an optimal model (\cref{eq:modelling_error}). If $\vsX$ is a $\G$-symmetric space, $\detDynMap[\vsX]$ is $\G$-equivariant.
    \label[proposition]{prop:optimal_models}
\end{proposition}
%
\begin{proof}
    Let $\G\state$ denote an arbitrary set of symmetric states and $\G\obsState(\state)$ represent their symmetric representations on $\vsX$. Consider a non $\G$-equivariant evolution map, denoted as $\bar{\Phi}_{\vsX}^{\dt}$. Such map will yield varying prediction errors for the states in $\G\state$. Consequently, it is possible to identify the state with the minimum prediction error, $\hat{\g} \Glact \state \in \G\state$, given $ \hat{\g} = \argmin_{\g \in \G} \dynRepError[\g \Glact \state,H](\obsState, \bar{\Phi}_{\vsX}^{\dt}) $, as per \cref{eq:modelling_error}. Subsequently, we can construct a new map that replicates the predicted evolution of $\hat{\g} \Glact \state$ for all symmetric states, defined as $ \detDynMap[\vsX](\obsState(\g \Glact \state)) := (\g \Gcomp \hat{\g}^{-1}) \Glact \bar{\Phi}_{\vsX}^{\dt}(\obsState(\hat{\g} \Glact \state)) \stforall \g \in \G$. It is important to note that $ \dynRepError[\g \Glact \state,H](\obsState, \detDynMap[\vsX]) \leq \dynRepError[\g \Glact \state,H](\obsState, \bar{\Phi}_{\vsX}^{\dt}) $ holds for all $\g \in \G$. By iteratively applying this process for all $\state \in \sfrac{\domain}{\G}$, the resulting map will exhibit $\G$-equivariance.
\end{proof}
%
Consistent with \cref{def:symmetric_dynamical_system}, we denote models $(\vsX, \detDynMap[\vsX])$ that possess a $\G$-symmetric modelling state space $\vsX$ and a $\G$-equivariant evolution map as $\G$-symmetric models. A familiar example is the Lagrangian model of rigid-body dynamics. Since, for $\G$-symmetric robotic systems (e.g., the mini-cheetah in \cref{fig:teaser}), the modelling space $\vsX=\confSpace \times \tangConfSpace$, defined by the space of generalized position $\confSpace$ and velocity $\tangConfSpace$ coordinates, is a symmetric vector space \citep[III]{Ordonez-Apraez-RSS-23}. The symmetry of the space is characterized by the group representation $\rep[\vsX]{} := \rep[\confSpace]{} \oplus \rep[\tangConfSpace]{}$, which describes the transformations shown in \cref{fig:teaser}-a. Furthermore, the evolution map $\detDynMap[\vsX]$, defined by the standard Euler-Lagrange equations of motion, features $\G$-equivariance \citep[VII.2]{lanczos2012variational}.

A key property of (non-linear or linear) $\G$-symmetric models is the decomposition of the modelling state space into $\isoCompNum$ isotypic subspaces $\vsX = \Oplus_{\isoCompIdx=1}^\isoCompNum \vsX_\isoCompIdx$ (refer to \cref{thm:iso}). This enables the projection of entire motion trajectories $(\obsState(\state_{t}))_{t\in \Time}$ onto each isotypic subspace $\obsState(\state_{t}) := \obsState^{(1)}(\state_{t}) \oplus^\perp \dots \oplus^\perp \obsState^{(\isoCompNum)}(\state_{t})$. Since each $\vsX_{\isoCompIdx}$ is a lower-dimensional space with a reduced number of symmetries, the trajectory's decomposition entails its characterization as a superposition of distinct synergistic motions $(\obsState^{(\isoCompIdx)}(\state_{t}))_{t\in \Time} \st \obsState^{(\isoCompIdx)}(\state_{t}) \in \vsX_{\isoCompIdx}$, each constrained to feature the subset of symmetries of $\vsX_{\isoCompIdx}$ (see \cref{fig:teaser}-b). This understanding is instrumental in characterizing different system behaviors, such as different locomotion gaits or manipulation tasks, through their lower-dimensional projections onto each isotypic subspace, as detailed in \cref{fig:teaser}-c.

For $\G$-symmetric linear dynamics models $(\vsZ,\mK_\dt)$, the isotypic decomposition (\cref{thm:iso}) of 
$
\vsZ 
= 
\Oplus_{\isoCompIdx=1}^\isoCompNum \vsZ_\isoCompIdx
= 
\Oplus_{\isoCompIdx=1}^\isoCompNum \Oplus_{\isoCompIrrepIdx=1}^{\irrepMultiplicity[\isoCompIdx]} \vsZ_{\isoCompIdx,\isoCompIrrepIdx} 
$ and the associated group representation $
\rep[\vsZ]{} 
= 
\Oplus_{\isoCompIdx=1}^\isoCompNum \rep[\vsZ_{\isoCompIdx}]{} =
\Oplus_{\isoCompIdx=1}^\isoCompNum \Oplus_{\isoCompIrrepIdx=1}^{\irrepMultiplicity[\isoCompIdx]}  \irrep[\isoCompIdx]{}
$ imply the dynamics' decomposition into $\isoCompNum$ linear subsystems $\mK_\dt = \Oplus_{\isoCompIdx=1}^\isoCompNum \mK_{\dt,\isoCompIdx}$. 
Each subsystem is $\G$-equivariant, $\rep[\vsZ_{\isoCompIdx}]{g}\mK_{\dt,\isoCompIdx} = \mK_{\dt,\isoCompIdx} \rep[\vsZ_{\isoCompIdx}]{g} \stforall \g \in \G$, and evolves the state projections into isotypic subspaces independently. Moreover, each subsystem is block-decomposed into scalar multiples of the identity map $\mI_\isoCompIdx: \R^{|\irrep[\isoCompIdx]{}|} \mapsto \R^{|\irrep[\isoCompIdx]{}|}$, such that 
\begin{equation}
    \small 
    \mK_{\dt} = 
    \mathrel{
        \mathsmaller{
        \begin{bsmallmatrix}
        \mK_{\dt,1} &  & \\ 
        & &\ddots & \\
        & & & \mK_{\dt,\isoCompNum}
        \end{bsmallmatrix}
        }
    }
    , \quad  
    \mK_{\dt,\isoCompIdx} = 
    \begin{bsmallmatrix}
        c_{1,1} \mI_\isoCompIdx & \dots  & c_{1,d_\isoCompIdx} \mI_\isoCompIdx \\
        \cdots & \cdots  & \cdots                                  \\ 
        c_{d_\isoCompIdx,1} \mI_\isoCompIdx & \dots  & c_{d_\isoCompIdx,d_\isoCompIdx} \mI_\isoCompIdx
    \end{bsmallmatrix},  
    \quad 
    c_{i,j} \in \R, 
    \; d_\isoCompIdx := \sfrac{|\vsZ_{\isoCompIdx}|}{|\irrep[\isoCompIdx]{}|}, \forall\; \isoCompIdx \in [1, \isoCompNum]
    \label{eq:isosubspace_lin_structure}
\end{equation}
The block-diagonal structure of $\mK_\dt$, and the block-structure of each $\mK_{\dt,\isoCompIdx}$ are geometric constraints that originate from Schur's Lemma (\cref{lemma:schursLemma}), a standard result in harmonic analysis. These properties, essentially stated in \citet[Thm 3.5]{golubitsky2012singularities_groups_bifurcation} for finite-dimensional spaces, can be generalized to Hilbert spaces, leading to the following result:

\begin{theorem}[Isotypic decomposition of symmetric linear models]\label{thm:dyn_iso}\\
    Let $(\vsZ, \oK_\dt)$ be a $\G$-symmetric linear model, $\G$ be a finite group, and $\vsZ = \Oplus_{\isoCompIdx=1}^\isoCompNum \vsZ_\isoCompIdx$ be the space isotypic decomposition. Then $\oK_\dt$ is block-diagonal $\oK_\dt = \Oplus_{\isoCompIdx=1}^\isoCompNum \oK_{\dt,i}$, where each
    $
        \oK_{\dt,\isoCompIdx}: \vsZ_{\isoCompIdx} \mapsto \vsZ_{\isoCompIdx}
    $ is a $\G$-equivariant linear map characterizing the independent dynamics of each isotypic subspace.
    \label{thm:lin_model_decomposition}
\end{theorem}

\noindent
This result has two primary applications in robot dynamics modelling.  The first involves decomposing local linear models resulting from local (or state-dependent) linearizations of nonlinear dynamical models. These are widely used in the iLQR \citep{li2004iLQR} and DDP \citep{mayne1966DDP} algorithms, which are instrumental for trajectory optimization \citep{tassa2014control,mastalli-crocoddyl}, and state estimation \citep{alessandri2003ddpestimation,kobilarov2015estimation,martinez2024inertiaest}. For these methods, the model decomposition facilitates the parallel optimization of the $\isoCompNum$ local linear sub-models, presenting a promising direction for future research. The second application pertains to the approximation of $\G$-symmetric Koopman models, which we discuss next.

\paragraph{Dynamic harmonic analysis}
For a $\G$-symmetric dynamical system $\detDynMap$, an injective state representation $\obsState: \domain \mapsto \vsZ$, and $\G$-symmetric modelling space $\vsZ$, the Koopman operator $\oK_\dt$ (\cref{eq:koopman}) is $\G$-equivariant and globally optimal by construction. Hence, by \cref{thm:dyn_iso}, it decomposes into $\isoCompNum$ Koopman operators characterizing the dynamics of each isotypic subspace. Furthermore, for each eigenpair $(\lambda, \psi)$ of $\oK_\dt$ and every $\g \in \G$, the symmetric function $\psi_{\g}:=\psi(\g\Glact(\cdot))\colon\domain\mapsto\C$ is also an eigenfunction of $\oK_\dt$ in the same eigenspace,
\begin{equation}
    \small
    \lambda\, \psi_{\g}(\state) =\lambda\, \psi(\g\Glact\state) = [\oK_\dt \, \psi](\g\Glact\state) = \psi(\detDynMap(\g\Glact\state)) = \psi(\g\Glact\detDynMap(\state)) = [\oK_\dt \, \psi_{\g}](\state),
    \label{eq:orbit_koopman_eigenfunctions}
\end{equation}
implying that the Koopman eigenspaces are $\G$-symmetric spaces. Thus, applying the isotypic decomposition to each eigenspace (\cref{thm:iso}) captures the relation of the temporal evolution of eigenfunctions with distinct symmetries (encoded by the \textsl{irreps}) via Koopman eigenvalues. 
This global decomposition of the dynamics in isotypic subspaces and its symmetry-aware spectral decomposition is referred to as \textit{dynamics harmonic analysis} (DHA).  As we show in \cref{sec:method,sec:experiments}, DHA can be leveraged to learn data-driven approximations of the $\G$-equivariant Koopman operator.


\section{{$\G$}-symmetric data-driven Koopman models}\label{sec:method}

The Koopman operator formalism, while practically unfeasible, has inspired numerous data-driven models aiming to approximate the infinite-dimensional operator in finite-dimensions \citep{brunton-koopman}. This process relies on a dataset of observations of state trajectories $\sD =\{(\obsState(\state_t))_{t=0}^{H}, \dots \}$ on a modelling space $\vsX \subseteq \R^{\dimModelSpace}$. Where $\obsState = [\obsFn_1,..., \obsFn_\dimModelSpace]: \domain \mapsto \vsX$ is a state representation composed of physical observable functions $\obsFn_i : \domain \mapsto \R$ that are measured or estimated from the state (e.g., position, momentum, energy). Then, the objective is to find a (latent) state representation vector-valued function $\latObsState = [\latObsFn_1,..., \latObsFn_\dimLatModelSpace]: \domain \to \widetilde{\vsZ} \subseteq \R^\dimLatModelSpace$, that spans a finite-dimensional space of functions $\obsSpace[\widetilde{\vsZ}]:=\{ \latObsFn_{\repObsFn}(\cdot):= \sum_{i=1}^{\dimLatModelSpace} \alpha_{i} \latObsFn_{i}(\cdot) = \innerprod{\latObsState(\cdot)}{\repObsFn}, \st \repObsFn\in\R^{\dimLatModelSpace}\}$ (see \cref{appendix:observable_functions_and_numerical_representation}), on which the Koopman operator $\oK_\dt$ is approximated by a matrix $\mK_\dt^*\in \C^{\dimLatModelSpace\times\dimLatModelSpace}$, given by
%
$
    (\oK_\dt \latObsFn_{\repObsFn})(\cdot) 
    \approx 
    \latObsFn_{\mK_\dt^* \repObsFn}(\cdot)
    := 
    \innerprod{\latObsState(\cdot)}{\mK_\dt^* \repObsFn} 
    = 
    \innerprod{\mK_\dt \latObsState(\cdot)}{\repObsFn}
$ \citep{kostic2022kernels}.
Where $\obsSpace[\widetilde{\vsZ}]$ represent the finite-dimensional approximation of the space of functions  $\obsSpace[\vsZ]$ on which $\oK_\dt$ is defined (see \cref{eq:koopman}), and $\repObsFn=[\alpha_1,\dots] \in \R^\dimLatModelSpace$ is the coefficient vector representation of $\latObsFn_\repObsFn$ in a basis of $\obsSpace[\widetilde{\vsZ}]$. 

Among different approaches to building Koopman data-driven models, we focus on the Dynamics Auto-Encoder (DAE) architecture \citep{lusch2018deep}. This model parameterizes the matrix approximating the Koopman operator $\mK_\dt^*$ as a trainable linear map, and the (latent) state representation $\latObsState(\cdot) := (\encoder \circ \obsState)(\cdot): \domain \mapsto \widetilde{\vsZ}$ with an encoder neural network $\encoder$. A decoder $\decoder: \widetilde{\vsZ} \mapsto \vsX$ is also defined to reconstruct states in the physical observable space. The cost function for DAE is composed of a reconstruction loss (encouraging injectivity of $\latObsState$) and a state prediction error in both physical observable $\vsX$ and latent $\widetilde{\vsZ}$ spaces (encouraging the minimization of the modelling error):
\begin{equation*}
    \small
    \loss(\state_t,H) 
    =
    \mathrel{\mathsmaller{
        \sum_{h=0}^{H}
        }
    }
    \ubcolor[awesomeorange]{
        ||
            \obsState(\state_{t+h\dt}) - \decoder(\mK_\dt^{h}\latObsState(\state_{t})) 
        ||^2
    }{\text{Reconstruction and} \; \dynRepError[\state_t,H](\obsState, \detDynMap[\vsX])}
    +
    \gamma
     \ubcolor{
        ||
            \latObsState(\state_{t+h\dt}) - \mK_\dt^{h} \latObsState(\state_{t})
        ||^2
    }{\dynRepError[\state_t,H](\latObsState, \mK_\dt)}
    \vspace*{-0.3cm}
\end{equation*}
\noindent
Where $\gamma$ balances the modelling errors in $\widetilde{\vsZ}$ and $\vsX$, and the prediction horizon is assumed $H \ll |\sD|$.
\paragraph{The Equivariant DAE (eDAE)} 
When modelling $\G$-symmetric dynamical systems (\cref{def:symmetric_dynamical_system}), the DAE architecture can be adapted to leverage the symmetry priors. 
First, to exploit the theoretical isotypic decomposition of the space of observable functions $\obsSpace[\vsZ]$ (\cref{prop:optimal_models,thm:iso}), we must ensure that $\obsSpace[\widetilde{\vsZ}]$ is a $\G$-symmetric function space, such that it can be decomposed in isotypic subspaces $\obsSpace[\widetilde{\vsZ}] := \Oplus_{\isoCompIdx=1}^\isoCompNum \obsSpace[\widetilde{\vsZ}_\isoCompIdx]$. This can be achieved by restricting the state representation $\latObsState: \domain \rightarrow \widetilde{\vsZ}$ to be a $\G$-equivariant map, 
such that both $\widetilde{\vsZ}$ and $\obsSpace[\widetilde{\vsZ}]$ are $\G$-symmetric spaces (see \cref{appendix:symmetric_function_spaces}). In practice, this is done by constraining the encoder to be a $\G$-equivariant neural network, such that:
\begin{equation}
    \small
    \g \Glact \latObsState(\state_t) 
    = 
    \rep[\widetilde{\vsZ}]{g} (\encoder \circ \obsState)(\state) 
    = 
    \encoder (\rep[\vsX]{g}\obsState(\state)) 
    = 
    \encoder(\obsState(\g\Glact\state)) 
    = 
    \latObsState(\g \Glact \state_t)
    \stforall \g \in \G, \state \in \domain.
\end{equation}
Where the group representation in the physical observable space $\rep[\vsX]{}$ is assumed to be known from prior knowledge, and the latent group representation $\rep[\widetilde{\vsZ}]{}$ is defined to be equivalent to the direct sum of $\sfrac{\dimLatModelSpace}{|G|} 
$ copies of the group regular representation, following \citep[thm. 1.12]{Knapp1986}.  

Then, to exploit the theoretical $\G$-equivariance of $\oK_\dt$ (\cref{prop:optimal_models,thm:lin_model_decomposition}), we parameterize $\mK_\dt$ as a $\G$-equivariant matrix, $\rep[\widetilde{\vsZ}]{\g}\mK_\dt = \mK_\dt \rep[\widetilde{\vsZ}]{\g} \stforall \g \in \G$. Furthermore, if the latent group representation is defined on the isotypic basis 
$
    \rep[\widetilde{\vsZ}]{}
    =
    \Oplus_{\isoCompIdx=1}^\isoCompNum \rep[{\widetilde{\vsZ}_{\isoCompIdx}}]{}
    = 
    \Oplus_{\isoCompIdx=1}^\isoCompNum \Oplus_{\isoCompIrrepIdx=1}^{\irrepMultiplicity[\isoCompIdx]} \irrep[\isoCompIdx]{}
$ (see \cref{thm:iso}), then the matrix $\mK_\dt$ decomposes in block-diagonal form $\mK_\dt = \Oplus_{\isoCompIdx=1}^\isoCompNum \mK_{\dt,\isoCompIdx}$, where each $\mK_{\dt,\isoCompIdx}$ is a $\G$-equivariant matrix characterizing the dynamics on the isotypic subspace $\widetilde{\vsZ}_{\isoCompIdx}$. This constraint on the learnable parameters of $\mK_\dt$ is relevant geometric prior as it (i) ensures the learned latent dynamics model is $\G$-symmetric, (ii) exploits the orthogonality between distinct isotypic subspaces (eliminating spurious correlations between dimensions of $\widetilde{\vsZ}$), and (iii) constraints the minimum dimension of the eigenspaces of each isotypic subspace operator $\mK_{\dt,\isoCompIdx}$ to match the dimension of the subspace irreducible representation dimension $|\irrep[\isoCompIdx]{}|$ (\cref{eq:isosubspace_lin_structure}). This ensures that eigenfunctions of each isotypic subspace $\obsSpace[\vsZ_\isoCompIdx]$ are appropiatedly approximated in groups of symmetric functions by their coefficient eigenvectors representations $\G\vv^{(\isoCompIdx)} = \{\g \Glact \vv^{(\isoCompIdx)} \st \mK_{\dt,\isoCompIdx}\vv^{(\isoCompIdx)} = \lambda^{(\isoCompIdx)}\vv^{(\isoCompIdx)}, \; \vv^{(\isoCompIdx)}\in \widetilde{\vsZ}_i\}$, with temporal dynamics governed by an eigenvalue $\lambda^{(\isoCompIdx)}$; refer to \cref{appendix:symmetric_function_spaces}.  

\vspace{-0.3cm}
\section{Experiments and results}\label{sec:experiments}
%
We conduct two experiments comparing equivariant and non-equivariant Koopman models using DAE and eDAE architectures. We set $\gamma= \sqrt{\sfrac{|\vsX|}{|\vsZ|}}$ in all experiments to balance the error across dimensions of $\vsX$ and $\vsZ$. Lastly, to test modelling error and generalization, we use testing datasets with trajectories uniformly sampled across all state space quotient sets $\domain / \G$ (see \cref{fig:teaser}-a).

\begin{figure}[t!]
    \centering
    \includegraphics[width=\textwidth]{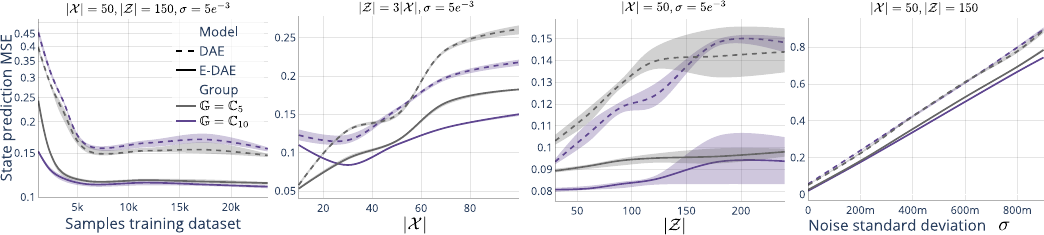}
    \vspace*{-0.5cm}
    \caption{
        Test set prediction mean square error (MSE) of learned Koopman models (DAE and eDAE) for synthetic systems with symmetry groups $\CyclicGroup[5]$ and $\CyclicGroup[10]$, varying state dimension $|\vsX|$, latent model space dimension $|\vsZ|$, and noise variance $\sigma$. Solid lines and shaded areas represent the mean, maximum, and minimum prediction error among $4$ training seeds. (a) MSE vs. training samples. (b) MSE over varying state dimensions. (c) MSE over varying dimensionality of the latent model space. (d) MSE over varying noise variance $\sigma$.
        \label{fig:linear_exp_results}
        \vspace{-1.0em}
    }
\end{figure}

\paragraph{Synthetic symmetric dynamical systems with finite state symmetry groups} This experiment models synthetic nonlinear $\G$-symmetric systems with arbitrary groups $\G$. The systems are constrained stable linear systems with stochastic perturbations $\obsState(\state_{t + \dt}) = \mA_\dt \obsState(\state_{t}) + \vepsilon_t, , \mC \state_{t} \geq \vc$, where $\obsState(\state) \in \vsX \subseteq \R^{\dimModelSpace}$ is the system state's numerical representation, $\mA_\dt \in \R^{\dimModelSpace \times \dimModelSpace}$ the linear dynamics matrix, $\vepsilon \in \R^\dimModelSpace$ a white-noise stochastic process with standard deviation $\sigma$, and $\mC \in \R^{n_c \times \dimModelSpace}, \vc \in \R^{n_c}$ the parameters of $n_c$ inequality hyper-plane constraints. These systems are $\G$-symmetric if $\mA_\dt$ is $\G$-equivariant $\rep[\vsX]{\g} \mA_\dt = \mA_\dt \rep[\vsX]{\g} \stforall \g \in \G$ and any constraint is also enforced for all symmetric states $\mC_{k,:}\;  \g \Glact \obsState(\state) \geq \vc_{k} \stforall k \in [1, n_c], \g \in \G$.

These synthetic systems let us assess symmetry exploitation in learning Koopman models for arbitrary groups $\G$, system's dimensionality $|\vsX|$, latent state dimensionality $|\vsZ|$, and noise standard deviation $\sigma$. The results show that the eDAE architecture provides superior models with better sample efficiency and generalization (\cref{fig:linear_exp_results}-a), reduced sensitivity to the dimensionality of $\vsX$ (\cref{fig:linear_exp_results}-b) and $\vsZ$ (\cref{fig:linear_exp_results}-c), and improved noise robustness (\cref{fig:linear_exp_results}-d).


\paragraph{Modelling quadruped closed-loop dynamics}
In this experiment, we investigate using a Koopman model for robot dynamics while quantifying the impact of symmetry exploitation. The focus is on modelling the closed-loop dynamics of the mini-cheetah quadruped robot's locomotion on mildly uneven terrain. The training data comprises a few motion trajectories executed by a $\G$-equivariant model predictive controller that tracks a desired target base velocity with a fixed trotting periodic gait \citep{amatucci}. As a result, the closed-loop dynamics are $\G$-equivariant \citep[eq.~(3-4)]{Ordonez-Apraez-RSS-23} and stable, with a limit-cycle trajectory describing the gait cycle and transient dynamics governed by the controller's correction for tracking errors.
The state of the system is numerically represented as $\obsState(\state_{t}) = [\q[js,t], \dq[js,t], z_{t}, \vo_{t}, \vv_{\text{err},t}, \vw_{\text{err},t}] \in \domain \subseteq \R^{37}$, composed of the joint-space generalized position $\q[js] \in \confSpace_{js} \subseteq \R^{12}$, and velocity $\dq[js] \in \tangConfSpace_{js} \subseteq \R^{12}$ coordinates, base height $z_t \in \R^1$, base orientation quaternion $\vo \in \R^4$, and the error in the desired linear and angular base velocities $\vv_{\text{err},t} \in \R^3$ and $\vw_{\text{err},t}\in \R^3$, respectively.

The symmetry group of the robot is $\G = \KleinFourGroup \times \CyclicGroup[2]$, of order $8$ (see \href{https://github.com/Danfoa/MorphoSymm/blob/devel/docs/static/animations/mini_cheetah-C2xC2xC2-symmetries_anim_static.gif}{symmetric states}). This group implies the decomposition of the system's state set into $8$ copies of the quotient set of unique states $\domain\, /\, \G$ (\cref{sec:symmetries_dynamical_systems}), and the isotypic decomposition of the physical modelling space $\vsX$ and the space of observable functions  $\obsSpace[\vsZ]$ into at most $8$ isotypic subspaces (\cref{thm:iso,thm:lin_model_decomposition}). Therefore, for this dynamical system exploiting the symmetry prior is crucial to mitigate the curse of dimensionality \citep{higgins2022symmetry} and biases of the training dataset (which is collected from trajectories originating from one of the $8$ quotient sets; see \href{https://danfoa.github.io/MorphoSymm/static/dynamic_animations/mini-cheetah_animation_C2xC2xC2.gif}{animation}). The results demonstrate superior performance of the $\G$-equivariant Koopman models (eDAE) over the models trained with data-augmentation DAE$_{aug}$, and without symmetry exploitation (DAE) in terms of sample-efficiency (\cref{fig:mini_cheetah_results}-a), forecasting error (\cref{fig:mini_cheetah_results}-b-d), and robustness to hyper-parameters variation (\cref{fig:mini_cheetah_results}-c). Furthermore, models exploiting the entire symmetry group $\KleinFourGroup \times \CyclicGroup[2]$ consistently outperform those exploiting only the subgroup $\KleinFourGroup$. 
The $\G$-symmetric eDAE models excel due to their ability to capture both transient and stable locomotion dynamics from an initial configuration $\state_0$, and to generalize to symmetric states $\G\state_0$, potentially unseen during training; refer to \cref{fig:teaser}-a and \href{https://danfoa.github.io/MorphoSymm/static/dynamic_animations/mini-cheetah_animation_C2xC2xC2.gif}{animation}. 
Moreover, by utilizing DHA, these models can decompose the learned latent linear dynamics into the linear dynamics of each isotypic subspace, thereby accurately capturing the distinct roles (and relevance) of isotypic subspaces in locomotion dynamics (see \cref{fig:teaser}-b,c).

\begin{figure}[t!]
    \centering\includegraphics[width=\textwidth]{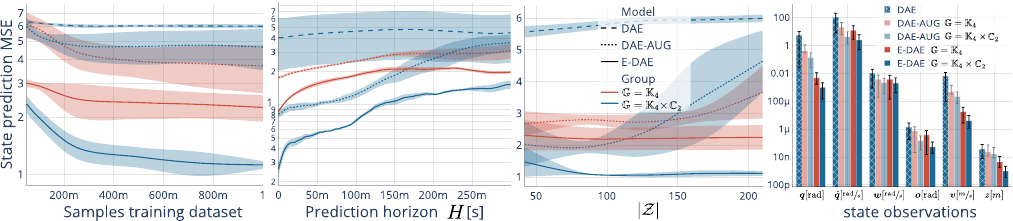}
    \vspace*{-0.7cm}
    \caption{
    The test prediction mean square error (MSE) of Koopman models (DAE, DAE$_{aug}$, eDAE) modelling the closed-loop dynamics of locomotion of the mini-cheetah robot. We compare models exploiting the system's full symmetry group $\G=\KleinFourGroup\times\CyclicGroup[2]$ and the subgroup $\G=\KleinFourGroup$. Solid lines and shaded areas show the mean, max, and min prediction error among $4$ training seeds. (a) MSE vs. training samples. (b) MSE vs. prediction horizon. (c) MSE for varying latent space dimensionality $|\vsZ|$. (d) MSE of measurable state observables in original units.
    \label{fig:mini_cheetah_results}
    \vspace*{-0.5cm}
    }
\end{figure}

\vspace{-0.3cm}
\section{Conclusions}
We introduced the use of harmonic analysis for decomposing and understanding the dynamics of symmetric robotic systems. By partitioning the state space into isotypic subspaces, we have shown how complex motions can be characterized as the superposition of lower-dimensional, symmetric and synergistic motions. This entails the decomposition of (local and global) linear models of the system's dynamics into independent models for each subspace. Leveraging this, we learn a data-driven global linear model using a novel equivariant deep-learning architecture to approximate the Koopman operator. The method's practical validity is evidenced by presenting the first successful attempt to learn a liner model of the closed-loop dynamics of a quadruped robot's locomotion.

\newpage
\acks{This work was supported in part by  PNRR MUR Project PE000013 CUP J53C22003010006 "Future Artificial Intelligence Research (hereafter FAIR)", funded by the European Union – NextGenerationEU}

\bibliography{references}


\newpage

\section*{Notation}

\begin{tabular*}{\textwidth}{@{\extracolsep{\fill}} p{1in}p{\dimexpr\textwidth-1in-2\tabcolsep\relax}}
  \multicolumn{2}{c}{\textbf{Numbers and Arrays}} \\
  $x$ & A scalar, or scalar function $x(\cdot) $              \\
  ${\vx}$ & A vector, or vector-valued function $\vx(\cdot)$  \\
  $\vx_1 \oplus \vx_2$ & Direct sum (stacking) of vectors, such that $\vx_1 \oplus \vx_2 := \begin{bsmallmatrix} \vx_1 \\ \vx_2 \end{bsmallmatrix}$\\
  $\mK$ & A matrix\\
  $\mA \oplus \mB$ & Direct sum of matrices, such that $\mA \oplus \mB := \begin{bsmallmatrix} \mA & \mO \\ \mO & \mB \end{bsmallmatrix}$\\
  $\oK$ & A linear operator \\
  $\oA \oplus \oB$ & Direct sum of linear operators, such that $\oA \oplus \oB := \begin{bsmallmatrix} \oA & \oO \\ \oO & \oB \end{bsmallmatrix}$\\ 
  $\mI$ & Identity matrix\\
  $\oI$ & Identity operator\\
  \multicolumn{2}{c}{} \\ 
  \multicolumn{2}{c}{\textbf{Sets, Vector Spaces, and Function spaces}} \\
  $\vsX, \vsZ, \vsH, \vsF$ & A vector/Hilbert space \\
  $\sI_{\vsX}$ & A basis set of the vector space $\vsX$\\
  $\R,\C$ & The set of real and complex numbers \\
  $\vsX \oplus \vsY$ & Direct sum of vector spaces $\vsX$ and $\vsY$ such that if $\vx \in \vsX$ and $\vy \in \vsY$, then $\vx \oplus \vy \in \vsX \oplus \vsY$\\ 
  $\vsF$ & A function space \\
  $f_\repObsFn \in \vsF$ & A function in the function space $\vsF$, represented with the coefficients $\repObsFn$ on a chosen basis $\sI_{\vsF} = \{\hat{f}_1, \dots \}$, such that $f_\repObsFn(\cdot) = \sum_{i=1}^{\dimModelSpace} \alpha_{i} \hat{f}_{i}(\cdot)$, given $\repObsFn = [\alpha_1, \dots]$.\\
  \multicolumn{2}{c}{} \\ 
  \multicolumn{2}{c}{\textbf{Group and representation theory}} \\
  $\G$ & A symmetry group\\
  $\g, \g_1, \g_a$ & A symmetry group element\\  
  $\g \Glact \vx$ & The (left) group action of $\g$ on $\vx$ defined by $\g \Glact \vx := \rep[\vsX]{g}\vsX$, for a chosen basis $\sI_{\vsX}$\\
  $\rep[\vsX]{} $ & A representation of the group $\G$ on the vector space $\vsX$, defined for a chosen basis $\sI_{\vsX}$ \\
  $\rep[\vsX]{\g}$ & Representation of the group element $\g$ on the vector space $\vsX$, defined for a chosen basis $\sI_{\vsX}$ \\
  $\rep[\vsX]{} \oplus \rep[\vsY]{} $ & Direct sum of group representations, 
  such that $\rep[\vsX]{} \oplus \rep[\vsY]{} := \begin{bsmallmatrix} \rep[\vsX]{} &  \\  & \rep[\vsY]{} \end{bsmallmatrix}$\\
  $\G \vx$ & The group orbit of $\vx$, defined as $\G\vx := \{ \g \Glact \vx \;|\; \g \in \G\}$\\
  $\G[a] \times \G[b]$ & Direct product of groups $\G[a]$ and $\G[b]$\\
  $\UGroup[\vsX]$ & Unitary group on the vector space $\vsX$\\
  $\GLGroup[\vsX]$ & General Linear group on the vector space $\vsX$\\
  $\CyclicGroup[n]$ & Cyclic group of order $n$\\
  $\KleinFourGroup$ & Klein four-group\\
\end{tabular*}

\appendix
\renewcommand{\thesection}{\Roman{section}}

\section{Appendix}

\subsection{Finite dimensional function spaces of observable functions}
\label{appendix:observable_functions_and_numerical_representation}
Given a set of observable functions $ \basisSet_{\vsX} = \{\obsFn_{1}, \dots, \obsFn_{\dimModelSpace} \st \obsFn_{i}:\domain \mapsto \R, \forall i \in [1,\dimModelSpace]\}$, we can interprete these functions as the components of a vector-valued function $\obsState = [\obsFn_{1}, \dots] : \domain \mapsto \R^{\dimModelSpace}$, that enables the numerical representation of the state $\state$ as a point in a finite-dimensional vector space $\obsState(\state) \in \vsX \subseteq \R^{\dimModelSpace}$. For the objective of our work, we will also interprete $\basisSet_{\vsX}$ as the basis set of a finite-dimensional function space $\vsF_\vsX: \domain \mapsto \R$, such that any observable function $\obsFn \in \vsF_\vsX$ is defined by the linear combination of the basis functions 
$\obsFn_{\repObsFn}(\state) := \innerprod{\obsState(\,\cdot\,)}{\repObsFn} = \sum_{i=1}^{\dimModelSpace} \alpha_{i} \obsFn_{i}(\state) = \obsState(\state)^\transpose \repObsFn$. 
Where $\repObsFn=[\alpha_1,\dots] \in \R^\dimModelSpace$ are the coefficients of $\obsFn$ in the basis of $\vsF_\vsX$, and the notation $\obsFn_{\repObsFn}(\,\cdot\,)$ highlight the relationship between the function $\obsFn$ and its coefficient vector representation $\repObsFn$. 

\subsubsection{Symmetries of the state representation} 
\label[appendix]{appendix:equivariant_state_representations}
\begin{wrapfigure}{r}{0.25\linewidth}
    \begin{equation}
        \small
        \homomorphismDiag
            {\state}
            {\g \Glact \state}
            {\obsState(\state)}
            {\g \Glact \obsState(\state)}
            {\obsState}
            {\g \in \G}
        \label{eq:state_representation_equivariant}
    \end{equation}
\end{wrapfigure}
When the dynamical system possess a state symmetry group $\G$ (\cref{def:symmetric_dynamical_system}), appropriate numerical representations of the state are constrained to be $\G$-equivariant vector-value functions (see \cref{eq:state_representation_equivariant,prop:optimal_models}). This ensures that the symmetry relationship between any state $\state \in \domain$ and its symmetric states $\G\state:=\{ \g \Glact \state \st \g\in\G\}$ is preserved in the representation space $\vsX$, such that $\G \obsState(\state) = \{\g \Glact \obsState(\state) \st \g \in \G \} \subset \vsX$. 

\subsubsection{Symmetric function spaces} 
\label[appendix]{appendix:symmetric_function_spaces}
When $\vsX$ is a $\G$-symmetric space, the group is defined to act on any chosen basis set of the space, including the observable functions $\basisSet_\vsX$. This, in turn, ensures that the finite-dimensional function space $\text{span}(\basisSet_\vsX) := \vsF_\vsX : \domain \mapsto \R$ features the symmtry group $\G$,
being the elements of the space $\G$-equivariant functions, i.e., $\vsF_\vsX = \{\obsFn \st \g \Glact \obsFn_{ \repObsFn}(\state) = \obsFn_{ \repObsFn}(\g^{-1} \Glact \state) = \obsFn_{ \g \Glact \repObsFn}(\state), \forall \g \in \G\}$  (see \cref{def:action_funct_space}). Where the notation $\g \Glact \obsFn_{ \repObsFn}(\state) = \obsFn_{ \g \Glact \repObsFn}(\state)$ describes the action of a symmetry transformation on a observable function, as a linear transformation on its coefficients  vector representation $\repObsFn$.

\subsection{Group and representation theory}

\begin{definition}[\textbf{Group action on a function space}] 
    \label{def:action_funct_space}
        The (left) action of a group $\G$ on the space of functions $\vsX: \domain \rightarrow \C$, where $\domain$ is a set with symmetry group $\G$, is defined as: 
        \begin{subequations}
            \begin{align}
                \mapping
                {(\Glact)}
                { \G \times \vsX}{\vsX}
                {(\g,\obsFn(\state))}
                {\g \Glact \obsFn(\state) \doteq \obsFn(\g^{-1} \Glact \state)}
            \label{eq:symmetry_action_on_f_space_associativity-a}
            \end{align}
            From an algebraic perspective, the action inversion (\href{https://math.stackexchange.com/questions/387266/group-action-on-vector-space-of-all-functions-g-to-mathbbc}{\textit{contragredient representation}}) emerges to ensure that the symmetry group in the function space is a homomorphism of the group in the domain $(\g_1 \Glact \g_2) \Glact \obsFn(\state) \doteq \obsFn((\g_1 \Glact \g_2)^{-1}\Glact \state)$. Which can be proven by a couple of algebraic steps:
            \begin{align}
            \small
                (\g_1 \Glact (\g_2 \Glact \obsFn))(\state) 
                = 
                    (\g_1 \Glact \obsFn_{\g_2})(\state) 
                = 
                    \g_2 \Glact \obsFn(\g_1^{-1}\state) 
                =
                    \obsFn((\g_2^{-1} \Glact \g_1^{-1}) \Glact \state) 
                = 
                    \obsFn((\g_1 \Glact \g_2)^{-1} \Glact \state)
            \label{eq:symmetry_action_on_f_space_associativity-b}
            \end{align}
        \end{subequations}
        From a geometric perspective, when $\vsX$ is a separable Hilbert space, each function can be associated with its vector of coefficients representation $\obsFn_{\repObsFn}(\,\cdot\,) := \sum_{i=1}^{\dimModelSpace} \alpha_{i} \obsFn_{i}(\,\cdot\,) = \obsState(\,\cdot\,)^\transpose \repObsFn$. Here, $\obsState=[\obsFn_1, \dots ]$ represents the basis functions of $\vsX$. As the function space is symmetric, the group $\G$ acts on the basis set, leading to a group representation acting on the basis functions $\g \Glact \obsState(\,\cdot\,) = \rep[\vsX]{\g} \obsState(\,\cdot\,)$. The unitary representation of the group $\G$ on the function space is denoted by $\rep[\vsX]{}: \G \to \UGroup[\vsX]$, which is an invertible matrix/operator. This representation enable us to interpret the symmetry transformations of the function space as point transformations, where the points are the function's coefficient vector representation $\repObsFn$, that is:
        \begin{align}
            \small
            \g \Glact \obsFn_{\repObsFn}(\,\cdot\,) & := \sum_{i=1}^{\dimModelSpace} \alpha_{i} \obsFn_{i}(\g^{-1} \Glact\,\cdot\,) \nonumber \\
            & = (\obsState(\g^{-1} \Glact \,\cdot\,))^\transpose  \repObsFn \nonumber \\
            & = (\g^{-1} \Glact \obsState(\,\cdot\,))^\transpose  \repObsFn \nonumber \\
            & = \obsState(\,\cdot\,)^\transpose  \g \Glact \repObsFn \nonumber \\
            & = \obsFn_{\g \Glact \repObsFn}(\,\cdot\,)
        \end{align}
            
    \end{definition} 

\begin{lemma}[Schur's Lemma for Unitary representations {\citep[Prop 1.5]{Knapp1986}}]
    \label{lemma:schursLemma}
    Consider two Hilbert spaces, $\HilbertSpace$ and $\HilbertSpace'$, each with their respective irreducible unitary representations, denoted as $\irrep[\HilbertSpace]{} : \G \rightarrow \UGroup[\HilbertSpace]$ and $\irrep[\HilbertSpace']{} : \G \rightarrow \UGroup[\HilbertSpace']$. Suppose $\equivLinMap: \HilbertSpace \rightarrow \HilbertSpace'$ is a linear equivariant operator such that $\irrep[\HilbertSpace']{}\equivLinMap = \equivLinMap\irrep[\HilbertSpace]{}$. If the irreducible representations are not equivalent, i.e., $\irrep[\HilbertSpace]{} \nsim \irrep[\HilbertSpace']{}$, then $\equivLinMap$ is the trivial (or zero) map. Conversely, if $\irrep[\HilbertSpace]{} \sim \irrep[\HilbertSpace']{}$, then $\equivLinMap$ is a constant multiple of an isomorphism (\cref{def:equivLinearMapsLong}). Denoting $\oI$ as the identity operator, this can be expressed as:  
    \begin{subequations}
        \begin{align}
            \irrep[\HilbertSpace]{} \nsim \irrep[\HilbertSpace']{} & \iff && \bm{0}_{\HilbertSpace'} = \equivLinMap \vh \st \forall \; \vh \in \HilbertSpace 
            \\ 
            \irrep[\HilbertSpace]{} \sim \irrep[\HilbertSpace']{} & \iff && \equivLinMap = \alpha \oU \st 
                \alpha \in \C, 
                \oU \cdot \oU^{H} = \oI 
            \\
            \irrep[\HilbertSpace]{} = \irrep[\HilbertSpace']{} & \iff && \equivLinMap = \alpha \oI  
        \end{align}
    \end{subequations}

    \noindent
    For intiution refeer to the following blog \href{https://terrytao.wordpress.com/2011/01/23/the-peter-weyl-theorem-and-non-abelian-fourier-analysis-on-compact-groups/}{post}
\end{lemma}

\begin{definition}[Group stable space \& Group irreducible stable spaces] 
    \label{def:G_stable_subspace}
    Let $\rho_{\vsX}: \G \rightarrow \UGroup[\vsX]$ be a unitary representation on the Hilbert space $\vsX$. A subspace $\vsX' \subseteq \vsX$ is said to be \highlight{$\G$-stable} if 
    \begin{equation}
        \rep[\vsX]{g} \vh \in \vsX' \quad \st \vh \in \vsX' \quad \forall \quad \vw \in W, \g \in \G. 
    \end{equation}
    If the only $\G$-stable subspaces of $\vsX'$ are $\vsX'$ itself and $\{\bm{0}\}$, the space is said to be an \highlight{irreducible $\G$-stable} space. 
\end{definition}

\begin{definition}[Homomorphism, Isomorphism and equivariant linear maps]
    \label{def:equivLinearMapsLong}
    Let $\G$ be a symmetry group and $\vsX$ and $\vsX'$ be two distinct symmetric Hilbert spaces endowed with unitary representations $\rep[\vsX]{}:\G \rightarrow \UGroup[\vsX]$ and $\rep[\vsX']{}: \G \rightarrow \UGroup[\vsX']$, respectively. 
    
    A linear map $\equivLinMap: \vsX \rightarrow \vsX'$ is said to be \highlight{$\G$-equivariant} if it commutes with the group representations: $\rep[\vsX']{\g} \equivLinMap = \equivLinMap \rep[\vsX]{\g} \st \forall \; \g \in \G$. The space of all $\G$-equivariant linear maps is refered to as the space of \highlight{homomorphisms} (structure preserving maps) and its denoted as $\homomorphism[\G]{\vsX}{\vsX'}$
    The spaces are said to be \highlight{isomorphic} if the $\G$-equivariant map is invertible. The space of all invertible $\G$-equivariant linear maps between $\vsX$ and $\vsX'$ is denoted as $\isomorphism[\G]{\vsX}{\vsX'} \subset \homomorphism[\G]{\vsX}{\vsX'}$.

    Graphically, the diagrams of a homomorphism and isomorphism between $\vsX$ and $\vsX'$ are:
    \begin{equation}
        \underbrace{
            \homomorphismDiag
                {\vsX}
                {\vsX}
                {\vsX'}
                {\vsX'}
                {\equivLinMap}
                {\rep[\vsX]{}}
            }_{Homomorphism} 
        \quad \equivLinMap \in \homomorphism{\vsX}{\vsX'}
        \qquad \qquad 
        \underbrace{\isomorphismDiag{\vsX}{\vsX'}{\rep[\vsX]{}}{\rep[\vsX']{}}{\equivLinMap}}_{Isomorphism} \quad \equivLinMap \in \isomorphism{\vsX}{\vsX'}
    \end{equation}
\end{definition}

\end{document}